\documentclass[letterpaper, 10 pt, conference]{ieeeconf}

\usepackage{mathtools}
\usepackage{amsmath,mleftright}
\usepackage{xparse}
\usepackage{comment}
\usepackage{amssymb}

\usepackage{amsthm}
\usepackage{hyperref}
\usepackage{graphicx}
\usepackage{amsmath}
\usepackage{amssymb}
\let\labelindent\relax
\usepackage{enumitem}
\usepackage{cite}
\usepackage{tabularx}
\usepackage[table]{xcolor}

\usepackage{subfig}

\newtheorem{lemma}{Lemma}%

\newtheorem{proposition}{Proposition}

\newtheorem{definition}{Definition}%

\hypersetup{
    colorlinks=true,
    linkcolor=darkgray,
    filecolor=magenta,      
    urlcolor=cyan,
    pdftitle={Overleaf Example},
    pdfpagemode=FullScreen,
     citecolor    = darkgray
    }

 	\definecolor{newgreen}{RGB}{0,127,0}
  \definecolor{newblue}{RGB}{0,0,127}

\usepackage{titlesec}

\begin{document}


\title{A Differentiable Distance Metric for Robotics Through Generalized Alternating Projection}

\author{
Vinicius M. Gonçalves,\quad
Shiqing Wei,\quad
Eduardo Malacarne S. de Souza,\quad\\
Krishnamurthy Prashanth,\quad
Anthony Tzes,\quad
and Farshad Khorrami
}

\maketitle
\begin{abstract} In many robotics applications, it is necessary to compute not only the distance between the robot and the environment, but also its derivative — for example, when using control barrier functions. However, since the traditional Euclidean distance is not differentiable, there is a need for alternative distance metrics that possess this property. Recently, a metric with guaranteed differentiability was proposed \cite{GoncalvesSmoothDistances}. This approach has some  important drawbacks, which we address in this paper. We provide much simpler
and practical expressions for the smooth projection  for general convex polytopes. Additionally, as opposed to \cite{GoncalvesSmoothDistances}, we ensure that the distance vanishes as the objects overlap. We show the efficacy of the approach in experimental results. Our proposed distance metric is publicly available through the Python-based simulation package [removed for anonymity].
\end{abstract} 

\section{Introduction}
Many obstacle avoidance techniques require the derivative of the distance between the robot and the environment. For example, the widely used Control Barrier Function (CBF) approach \cite{ames2019cbf} incorporates an inequality constraint into an optimization problem to enforce obstacle avoidance. This constraint may involve the $k$-th derivative of the distance function with respect to the state variable if the system is of $k$-th order. However, it is well known that the traditional Euclidean distance between two general objects is not even once differentiable \cite{EscandeSCH}. Moreover, even in cases where the distance function is differentiable, its higher-order derivatives can be extremely large, making it unsuitable for control applications. This can lead to excessively high control inputs or inputs that vary too quickly (e.g., high accelerations in case the input is velocity). In either case, it may be infeasible to apply on a real robot \cite{GoncalvesSmoothDistances}.

Many works have addressed this issue. Some propose specific geometric shapes for which the derivatives of the Euclidean distance are well-defined \cite{Escande2, capsule, EscandeSCH, tracy2022diffpills}, while others adopt alternative distance metrics that ensure differentiability \cite{pseudodistance, pseudodist, SCHMEIER201567, article, diff, randomizedgradient, GoncalvesSmoothDistances, bssm, DiffOcclusion}. A recent review of approaches tackling the differentiability problem can be found in \cite{GoncalvesSmoothDistances}.

In \cite{GoncalvesSmoothDistances}, a smooth version of the half-squared distance between two convex sets is presented\footnote{As explained in \cite{GoncalvesSmoothDistances}, we use half-squared distances instead of standard distances because it is more mathematically convenient.}. The core idea behind this proposal is the definition of a smooth version of the half-squared \emph{point-to-set} function for a set $\mathcal{S}$, denoted as $E^{\mathcal{S}}(p)$, which is based on the computation of an integral. From this, by analogy with the true Euclidean half-squared distance, the authors define the \emph{smooth projection} of the point $p$ onto $\mathcal{S}$ as $\Pi^{\mathcal{S}}(p) \triangleq p - \frac{\partial E^{\mathcal{S}}}{\partial p}(p)$. The authors then prove that the classical Von Neumann's alternating projection algorithm (which projects onto both sets iteratively until convergence) guarantees convergence when we replace the true Euclidean projection with the smooth projection. The resulting converged smooth witness points, $a^*$ and $b^*$, can then be used in a function to compute the smooth distance between the two sets.

This approach has two important drawbacks. First, computing the projections is difficult, requiring the computation of complex integrals even for simple shapes such as boxes. Second, it lacks an important property: the distance does not vanish when two objects overlap. While the distance decreases as objects get closer, it remains a small positive value even when they intersect. Moreover, this residual value is difficult to determine \emph{a priori}, as it depends on the specific nature of the overlap. As a result, it is not straightforward to simply ``subtract out'' this positive offset from the smooth distance.

In this paper, we build upon \cite{GoncalvesSmoothDistances} to remove these limitations. For the first problem, we demonstrate that the integral can be eliminated by identifying the essential properties necessary for convergence. This  provides much simpler and practical expressions for the smooth projection, not only for the objects considered previously but also for general convex polytopes, which would be very difficult using \cite{GoncalvesSmoothDistances}. For the second problem, we introduce an additional property that ensures the distance vanishes when the objects overlap, which was absent in \cite{GoncalvesSmoothDistances}. Overall, the proposed algorithm is simple to implement, and it does not require any specialized solvers. Beyond providing an algorithm, our approach also contributes to the theoretical understanding of smooth distance metrics. In particular, our work leverages a game-theoretic interpretation of the proposed distance in \cite{GoncalvesSmoothDistances} - which was absent from the previous paper - to obtain some properties.

We have integrated our implementation into the package [removed for anonymity], which is available for installation online\footnote{The real name of the package is not provided per RAL's anonymity policy. In this paper's final version, the real package and the link for the code will be publicized here.}. We conducted tests to evaluate the convergence properties of our algorithm and carried out experiments on a Franka Emika robot by incorporating our proposed metric into a CBF-framework. All code used for comparisons and experiments, implemented in this package, is publicly available. Finally, for the sake of readability, all proofs are provided in the Appendix at the end of the paper.

\subsection{Mathematical notation}

If $u$ is a vector, $\|u\|$ represents the Euclidean distance, and if $M$ is a matrix, $\|M\|$ is the spectral norm (i.e., the square root of the largest eigenvalue of $M^\top M$). $I_{n \times n}$ represents the identity matrix of order $n$, whereas $0_{n \times m}$ represents the zero matrix of order $n \times m$. If $M$ and $N$ are symmetric matrices, $M \geq N$ (resp. $M > N$) should be interpreted as $M-N$ being positive semidefinite (resp. positive definite).
All vectors are column vectors. If $f: \mathbb{R}^n \mapsto \mathbb{R}$, $\frac{\partial f}{\partial p}$ (the gradient) is a column vector.  A function $\rho: \mathbb{R}^n \mapsto \mathbb{R}$ that is at least twice differentiable is said to be \emph{strictly convex} \footnote{We use in this paper a slightly more restrictive definition of ``strictly convex'' than the traditional definition, in which \(f(x)=x^4\), for example, would be strictly convex.} 
in a set $\mathcal{P}$ if $\frac{\partial^2 \rho}{\partial p^2}(p) > 0$ for all $p \in \mathcal{P}$. A set is said to be \emph{regular} if it is compact and has interior points.

A function $F: \mathbb{R}^n \mapsto \mathbb{R}^n$ is said to be 
\emph{globally contractive} if there exist $0 < C < 1$ such that for any $u, v \in \mathbb{R}^n$, $\|F(u)-F(v)\| \leq C \|u-v\|$.  If $F$ is differentiable, it is said to be \emph{locally contractible} if there exist $0 < c < 1$ such that $\left\|\frac{\partial F}{\partial p}(p) \right\| < c$  for all $p$. Being  locally contractible everywhere implies being globally contractible \cite{contmap}.

\section{Main results}

\subsection{Generalized Alternating Projection Algorithm}

Consider the following definition.

\begin{definition} Let $\mathcal{S} \subseteq \mathbb{R}^n$ be a convex regular set. \label{def:kp2s} For $k \geq 2$, a function $E^{\mathcal{S}}: \mathbb{R}^n \mapsto \mathbb{R}$ is said to be a \emph{\emph{$k^{th}$ order generalized point-to-set half squared metric} (\textbf{k-P2S} henceforth)} if the following holds:

\begin{enumerate}
    \item $E^{\mathcal{S}}(p) \geq 0$ ;
    \item $I_{n \times n} > \frac{\partial^2 E^{\mathcal{S}}}{\partial p^2}(p) > 0_{n \times n} $ iff $p \not \in \mathcal{S}$ and $\frac{\partial^2 E^{\mathcal{S}}}{\partial p^2}(p) = 0$ if $p \in \mathcal{S}$;
    \item $E^{\mathcal{S}}(p)$ is $k$-times differentiable on $p$;
    \item$E^{\mathcal{S}}(p)$ vanishes iff $p \in \mathcal{S}$;

\end{enumerate}

Furthermore, given this function $E^{\mathcal{S}}$, we define the \emph{$k^{th}$ order generalized projection} as the function $\Pi^{\mathcal{S}} : \mathbb{R}^n \mapsto \mathbb{R}^n$ given by $\Pi^{\mathcal{S}}(p) \triangleq p - \frac{\partial E^{\mathcal{S}}}{\partial p}(p)$. $\square$
\end{definition}

Note that, as it could have been expected from a ``traditional'' projection, there is no guarantee that $\Pi^{\mathcal{S}}(p) \in \mathcal{S}$ for all $p$. However, as a consequence of Property (4), this is true when $p \in \mathcal{S}$. This is a fundamental result for our approach.

\begin{proposition} \label{prop:simple}
    A k-P2S function has $\frac{\partial E^{\mathcal{S}}}{\partial p}(p) =0$ iff $p \in \mathcal{S}$. Consequently, $\Pi^{\mathcal{S}}(p) = p$ iff $p \in \mathcal{S}$.
\end{proposition}

We recall that all proofs are listed in the Appendix. Our proposed ``distance'', for clarity reasons, will be referred as \emph{metric}. Furthermore, as opposed to \cite{GoncalvesSmoothDistances}, our metric is not necessarily smooth (i.e., infinitely differentiable), but only differentiable a finite number of times. Thus, henceforth, we will denote it a \emph{differentiable metric}.

Property (4), and its consequence in Proposition \ref{prop:simple}, is essential to guarantee that the set-to-set differentiable metric vanishes when the objects overlap, which \cite{GoncalvesSmoothDistances} lacks.  Indeed, the half squared metric computed by the integral presented in \cite{GoncalvesSmoothDistances}
\begin{equation}
    E^{\mathcal{S}}(p) =  -h^2 \ln\left(\frac{1}{\textsl{Vol}(\mathcal{S})} \int_{\mathcal{S}} e^{-\|p-s\|^2/(2h^2)}ds\right)
\end{equation}
\noindent in which $\textsl{Vol}(\mathcal{S})$ is the $n$-dimensional volume of $\mathcal{S}$ and $h>0$ is a smoothing parameter, has the Properties (1), (2) and (3) (with $k = \infty$)  but it lacks Property (4). 

We proceed by establishing the generalized alternating algorithm when generalized projections, as in Definition \ref{def:kp2s}, are used instead of the classical ones. 

\begin{proposition} \label{prop:gan}  \emph{(Generalized alternating algorithm)} Let $\mathcal{A}$ and $\mathcal{B}$ be  two regular convex sets in which  $\mathcal{A} \cap \mathcal{B} = \emptyset$.  Consider one generalized  projection $\Pi^{\mathcal{A}}$ and $\Pi^{\mathcal{B}}$ for each one of them. Then, for any initial condition $a[0] \in \mathbb{R}^n$, the sequence:
\begin{equation}
\label{eq:gam}
    a[k+1] = \Pi^{\mathcal{A}}\big(\Pi^{\mathcal{B}}(a[k])\big)
\end{equation}
\noindent converges to a point $a^*$. Furthermore, this point is unique. $\square$
\end{proposition}

Note that the case in which the objects overlap, $\mathcal{A} \cap \mathcal{B} \not= \emptyset$, is not considered by this proposition. Experimental results shown that the algorithm converges 
even in this case, but the proof seems more complex. This is unnecessary, however. 
As it will be shown, we can define $a^*$ and $b^*$ in this case to be simply any pair of points such that $a^* = b^* \in \mathcal{A} \cap \mathcal{B}$. 

\begin{definition} \label{def:smowp}
    Let $\mathcal{A} \cap \mathcal{B} = \emptyset$. We define as the pair $(a^*,b^*)$, called \emph{differentiable witness points}, the pair formed by \emph{the} (since Proposition \ref{prop:gan} guarantees uniqueness) limit point of \eqref{eq:gam} $a^*$ and its respective counterpart $b^* = \Pi^{\mathcal{B}}(a^*)$. If $\mathcal{A} \cap \mathcal{B} \not= \emptyset$, we simply pick $(a^*,b^*)$ as any points in which $a^* = b^* \in \mathcal{A} \cap \mathcal{B}$.  $\square$
\end{definition}

 Note that when $\mathcal{A} \cap \mathcal{B} \not= \emptyset$, the differentiable witness points can be found by any classical convex distance calculating algorithm, as the traditional alternating algorithm \cite{vonneumann} or GJK \cite{gilbert2002fast}. We proceed by, as in \cite{GoncalvesSmoothDistances}, using the differentiable witness point to define a differentiable metric between sets.

\begin{definition} \label{def:ks2s} Given two regular convex sets $\mathcal{A}$ and $\mathcal{B}$, let ($a^*$,$b^*$) be differentiable witness points. Then, define the $k^{th}$  \emph{order generalized set-to-set half squared metric (\textbf{k-S2S} henceforth)} as:
\begin{equation}
    \Lambda^{\mathcal{A},\mathcal{B}} \triangleq E^{\mathcal{A}}(b^*)+E^{\mathcal{B}}(a^*)-\frac{\|a^*-b^*\|^2}{2}.
\end{equation}
$\square$
\end{definition}

Note that, from Definitions  \ref{def:kp2s}, \ref{def:smowp} and  \ref{def:ks2s}, when $\mathcal{A} \cap \mathcal{B} \not= \emptyset$, $\Lambda^{\mathcal{A},\mathcal{B}} = 0$ (because $ E^{\mathcal{A}}(b^*) = E^{\mathcal{B}}(a^*) = 0$ and $a^*=b^*$). However, although $(a^*,b^*)$ was defined in a ``discontinuous'' way (from the generalized alternating algorithm when $\mathcal{A} \cap \mathcal{B} = \emptyset$ and from the traditional witness points otherwise), it turns out that $ \Lambda^{\mathcal{A},\mathcal{B}}$ is ``continuous'', because in the limit case in which the objects overlap, $ \Lambda^{\mathcal{A},\mathcal{B}}$ goes to $0$.

\begin{proposition} \label{prop:pos} If $\mathcal{A} \cap \mathcal{B} = \emptyset$, $\Lambda^{\mathcal{A},\mathcal{B}}$ is positive. Also, in the limit case in which the two objects overlap, it goes to $0$.$\square$
\end{proposition}

The differentiability of $ \Lambda^{\mathcal{A},\mathcal{B}}$, when the two objects do not overlap, is guaranteed in the following result.

\begin{proposition} \label{prop:diff} Let $\tau \in \mathbb{R}$. Consider two $k$-times differentiable rigid transformations $T_A, T_B: \mathbb{R}^n \times \mathbb{R} \mapsto \mathbb{R}^n$ in the variable $\tau$, and extend this transformation to a set $\mathcal{P} \subseteq \mathbb{R}^n$ as $T(\mathcal{P},\tau) = \{ T(p,\tau) \ | \ p \in \mathcal{P}\}$. Suppose the regular convex sets $\mathcal{A}$ and $\mathcal{B}$ are moving according to $\mathcal{A}(\tau) = T_A(\mathcal{A}_0, \tau)$  and $\mathcal{B}(\tau) = T_B(\mathcal{B}_0,\tau)$ for two fixed regular convex sets $\mathcal{A}_0$ and $\mathcal{B}_0$ . Suppose $\Lambda^{\mathcal{A}(\tau),\mathcal{B}(\tau)}$ is built using two k-P2S. Then $\phi(\tau) = \Lambda^{\mathcal{A}(\tau),\mathcal{B}(\tau)}$ is $k$-times differentiable in the variable $\tau$ as long as $\mathcal{A}(\tau) \cap \mathcal{B}(\tau) = \emptyset$. $\square$
\end{proposition}

The result is established if \( \Lambda^{\mathcal{A},\mathcal{B}} \) depends on a single variable \( \tau \). However, if it depends on multiple variables, we can apply Proposition \ref{prop:diff} to each of these variables, provided the necessary conditions are met.

\section{Creating k-P2S functions}

The construction of \( k \)-P2S functions for regular convex polytopes \( \mathcal{S} \) is outlined in this section. Property (2) in Definition \ref{def:kp2s} implies that \( E^{\mathcal{S}}(p) \) should be strictly convex, i.e., \( \frac{\partial^2 E^{\mathcal{S}}}{\partial p^2}(p) > 0 \). Our construction begins by defining a function that satisfies all the properties in Definition \ref{def:kp2s}, except for this one, which will hold with \( \geq 0 \) instead of \( > 0 \) (thus ensuring normal convexity rather than strict convexity). We will then explain how to deform this function to preserve the other properties and achieve strict convexity. We will refer to this preliminary function as \textbf{weak k-P2S functions} and use \( e^{\mathcal{S}}(p) \) instead of \( E^{\mathcal{S}}(p) \) to represent them.

\subsection{Creating weak k-P2S functions for regular convex polytopes}

Consider the following definition.
\begin{definition} \label{def:k2ps-base} For $k \geq 2$, a function $\Phi: \mathbb{R} \mapsto \mathbb{R}$ is said to be a \emph{basic} k-P2S function if it has the following properties:
\begin{enumerate}
    \item $\Phi(s) \geq 0$;
    \item $1 > \Phi''(s) > 0$ if $s \geq 0$ and $\Phi''(s) = 0$ for $s \leq 0$;
    \item $\Phi(s)$ is $k$-times differentiable on $s$;
    \item $\Phi(s) = 0$ iff $s \leq 0$. 
\end{enumerate}

$\square$ \end{definition}

A base k-P2S function is ``essentially'' a k-P2S function (as in Definition \ref{def:kp2s}) for the one-dimensional set $(-\infty,0]$, with the difference that this set is non-compact (as the regularity condition in Definition \ref{def:kp2s} requires).

\begin{figure}[b]
\includegraphics[width=8cm, height=3cm]{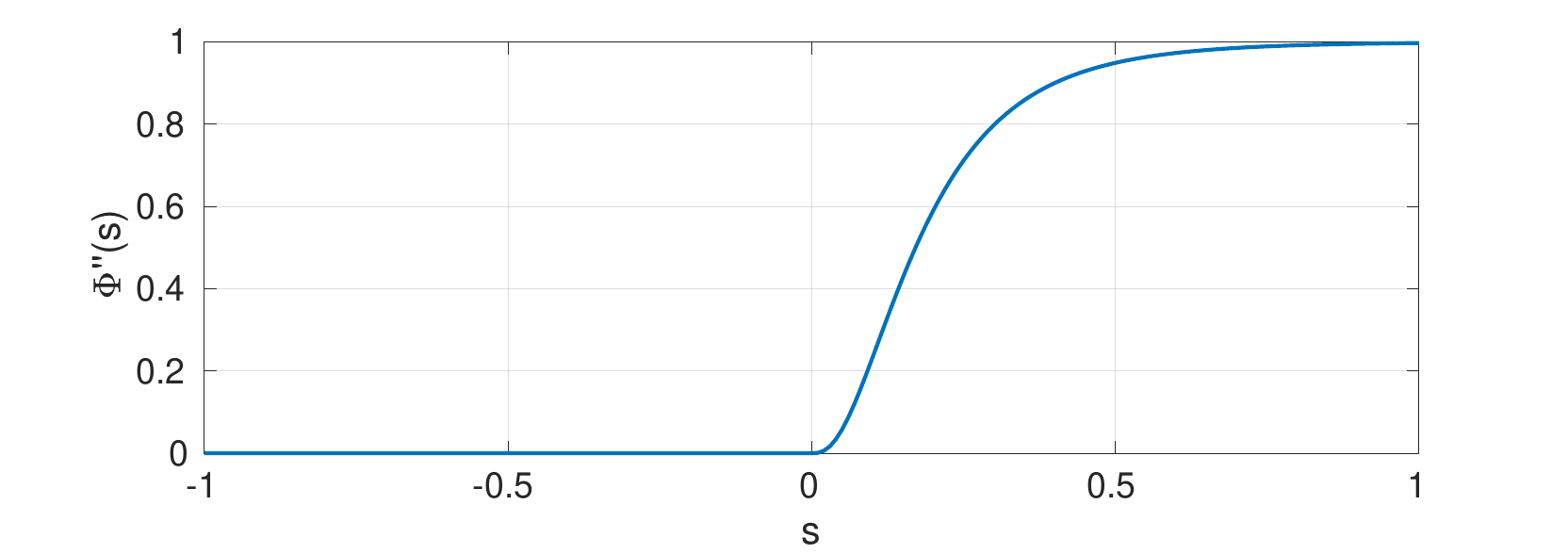}
\caption{The second derivative of $\Phi(s)$ as in \eqref{eq:Phifun}, that is, the integrand in \eqref{eq:Phifun}, for $h=0.1$ and $k=3$.}
\label{fig:figPhi}
\end{figure}

This function will serve as a base (hence the name) for constructing more complex (weak) \(k\)-P2S functions. Obtaining base \(k\)-P2S functions is not difficult. Intuitively, the second derivative of \( \Phi \) must resemble a differentiable version of the Heaviside indicator function that is $k-2$ times differentiable at $s=0$ (see Figure \ref{fig:figPhi}). Hence, one can use, for example:
\begin{eqnarray}
\label{eq:Phifun}
   &&\Phi(s) =  \int_0^s \int_0^\xi  \left( 1 - (r+1)^{-1/h} \right)^{k-1} \, dr \, d\xi =  \nonumber \\
   && \sum_{i=0}^{k-1} \binom{k-1}{i}\frac{(-1)^{k-1-i}}{1-i/h}\left[ \frac{(s+1)^{2-i/h}-1}{2-i/h} - s \right]
\end{eqnarray}

\noindent for \( s \geq 0 \) and \( 0 \) otherwise. Here, \( h \) is a positive smoothing parameter and $k \geq 2$ is an integer. Generally, the smaller is the $h$ and the larger is $k$, the smoother the metric will be. However, the convergence of the algorithm will become slower.

Consider the regular convex polytope \( \mathcal{S} \) as the intersection of half-spaces, i.e., \( \mathcal{S} = \{ p \in \mathbb{R}^n \mid u_i^{\top} p + v_i \leq 0, \, i=1,...,n \} \). We can assume, without loss of generality, that \( \|u_i\| = 1 \). A function that can serve as a weak \( k \)-P2S for \( \mathcal{S} \) is thus:

\begin{equation}
\label{eq:e}
    e^{\mathcal{S}}(p) =  \sum_{i=1}^n W_i\Phi\big(u_i^{\top}p+v_i\big)
\end{equation}

\noindent in which $W_i$ are positive constants. Each property in Definition \ref{def:kp2s} will be studied separately.

\textbf{Property (1):} This easily follows from the fact that \( W_i > 0 \) and Property (1) in Definition \ref{def:k2ps-base}.

\textbf{Property (2):} For Property (2), noting that a \emph{weak} \( k \)-P2S should satisfy \( I_{n \times n} > \frac{\partial^2 e^{\mathcal{S}}}{\partial p^2} \geq 0_{n \times n} \), we observe that the Hessian of \( e^{\mathcal{S}}(p) \) is:

\[
    \frac{\partial^2 e^{\mathcal{S}}}{\partial p^2}(p) =  \sum_{i=1}^n W_i \Phi''\left( u_i^{\top} p + v_i \right) u_i u_i^{\top}
\]

\noindent which is a positive semidefinite matrix because \( \Phi''(u_i^{\top} p + v_i) > 0 \) (see Property (2) of Definition \ref{def:k2ps-base}), \( W_i \) are nonnegative scalars, and each rank-1 matrix \( u_i u_i^{\top} \) is positive semidefinite. Since this is a positive semidefinite matrix, for it to have all its eigenvalues less than 1, it suffices that its spectral norm is less than 1. A simple bound for the spectral norm using the triangle inequality gives:

\[
    \left \| \frac{\partial^2 e^{\mathcal{S}}}{\partial p^2}(p) \right\| \leq  \sum_{i=1}^n W_i \Phi''\left( u_i^{\top} p + v_i \right)
\]

\noindent in which we used the fact that \( \Phi''(s) \geq 0 \) and, since \( \| u_i \| = 1 \), \( \| u_i u_i^{\top} \| = 1 \). Thus, it suffices that the right-hand side is less than 1. Each \( \Phi'' \) is bounded by 1 (see Property (2) of Definition \ref{def:k2ps-base}), so we could choose \( W_i = \frac{1}{n + \epsilon} \) for a small positive number \( \epsilon \), for all \( i \). However, it is beneficial to have the largest \( W_i \) possible: small \( W_i \)'s make the function \( \Phi \) small and will hinder the convergence of the iterative projection algorithm in \eqref{eq:gam}. Therefore, one could select \( W_i = \frac{1}{m + \epsilon} \), where \( m \) is the maximum number of inequalities \( u_i^{\top} p + v_i \) that can be positive simultaneously for all \( p \). The problem of finding the maximum number of linear inequalities that can be positive at the same time can be cast as a mixed linear integer program and solved offline.

\textbf{Property (3):} This follows directly from the fact that \( \Phi \) is \( k \)-times differentiable.

\textbf{Property (4):} For \( e^{\mathcal{S}}(p) = 0 \) if and only if \( p \in \mathcal{S} \), we note that each term \( W_i \Phi(u_i^{\top} p + v_i) \) is nonnegative. Thus, the sum only vanishes when each term vanishes. Since \( W_i > 0 \), this occurs only when all \( \Phi(u_i^{\top} p + v_i) \) vanish. But \( \Phi(s) \) vanishes only when \( s \leq 0 \), which implies that \( u_i^{\top} p + v_i \leq 0 \).

\subsection{Obtaining k-P2S functions from weak k-P2S functions}

To obtain a convex function $E^{\mathcal{S}}(p)$ as a k-P2S function from a weak k-P2S function $e^{\mathcal{S}}(p)$, we use the following result.

\begin{proposition} \label{prop:convexified} Let $\mathcal{S}$ be a regular convex set. Let $\rho: \mathbb{R}^n \mapsto \mathbb{R}$ be a $k$-times differentiable strictly convex function with $\|\frac{\partial^2 \rho}{\partial p^2}\| < 1$ and that is negative when evaluated at any point of $\mathcal{S}$. Furthermore, that $\rho$ has vanishing gradient only if $p \in \mathcal{S}$. Let $e^{\mathcal{S}}$ be a weak k-P2S function for the set $\mathcal{S}$. Then, there exists positive scalars $\varepsilon $ and $\sigma$ such that:

\begin{equation}
\label{eq:Efrome}
    E^{\mathcal{S}}(p) = \varepsilon \rho(p) + \sqrt{ \sigma^2 e^{\mathcal{S}}(p)^2 + \varepsilon ^2 \rho(p)^2}
\end{equation}
\noindent is a k-P2S function for the set $\mathcal{S}$ 
    
\end{proposition}

One example of function $\rho$ that satisfies the requirements is  $\rho(p) = 0.5(\|p-p_c\|^2-R^2)$, in which $p_c$ and $R$ are such that the ball $\{ p \in \mathbb{R}^n \ \| \ \|p-p_c\| \leq R\}$ strictly covers $\mathcal{S}$, which is possible because $\mathcal{S}$ is regular and thus compact. It is strictly convex, since its Hessian is $\varepsilon I_{n \times n}$, and it evaluates to a negative number inside the set $\mathcal{S}$, since any point in the set is inside the ball, and thus $\|p-p_c\|^2-R^2 \leq 0$. Furthermore, the gradient of $\rho$ vanishes only at $p=p_c \in \mathcal{S}$.

There does not exist a simple procedure to choose $\varepsilon$ and $\sigma$. The issue is that choosing very small $\sigma$ and $\epsilon$ creates a valid k-P2S function, but with a very small Hessian, and thus the convergence of the generalized alternating algorithm is slow. On the other hand, if they are very large, the condition $I_{n \times n} > \frac{\partial^2 E^{\mathcal{S}}}{\partial p^2}(p) $ (Property (2) in Definition \ref{def:kp2s}) may be violated.  This parameters can be obtained through experimentation: select one and test random points $p$ in $\mathbb{R}^n$ to check if $I_{n \times n} > \frac{\partial^2 E^{\mathcal{S}}}{\partial p^2}(p) $ holds.

Figure \ref{fig:convexified} shows level sets for $e^{\mathcal{S}}(p)$ and the respective function $E^{\mathcal{S}}(p)$ constructed from Lemma \ref{prop:convexified} with a function $\rho(p) =  0.5(\|p-p_c\|^2-R^2$). In this case, $\mathcal{S}$ (the zero level set of both functions) is a pentagon. Note that in $e^{\mathcal{S}}(p)$ the sublevel sets\footnote{The $c$ sublevel set of a function $f: \mathbb{R}^n \mapsto \mathbb{R}$ is $\{p \in \mathbb{R}^n \ | \ f(p) \leq c\}$.} are not strictly convex shapes (having ``flat'' sides), because $e^{\mathcal{S}}(p)$ is not strictly convex. The transformation in \eqref{eq:Efrome} ``bulges'' the sublevel sets (except the zero sublevel set) and transforms them into strictly convex shapes, because $E^{\mathcal{S}}(p)$ is strictly convex.

\section{Experiments}

We will show two experiments to highlights aspects of the proposed methodology. These experiments can be reproduced by running the files in \cite{uaibot_content}. Experimental data is also included in this folder. The installation of the aforementioned software package is necessary.

\begin{figure}[h]
\includegraphics[trim={4cm 0 0 1cm},clip, width=8.8cm, height=3.5cm]{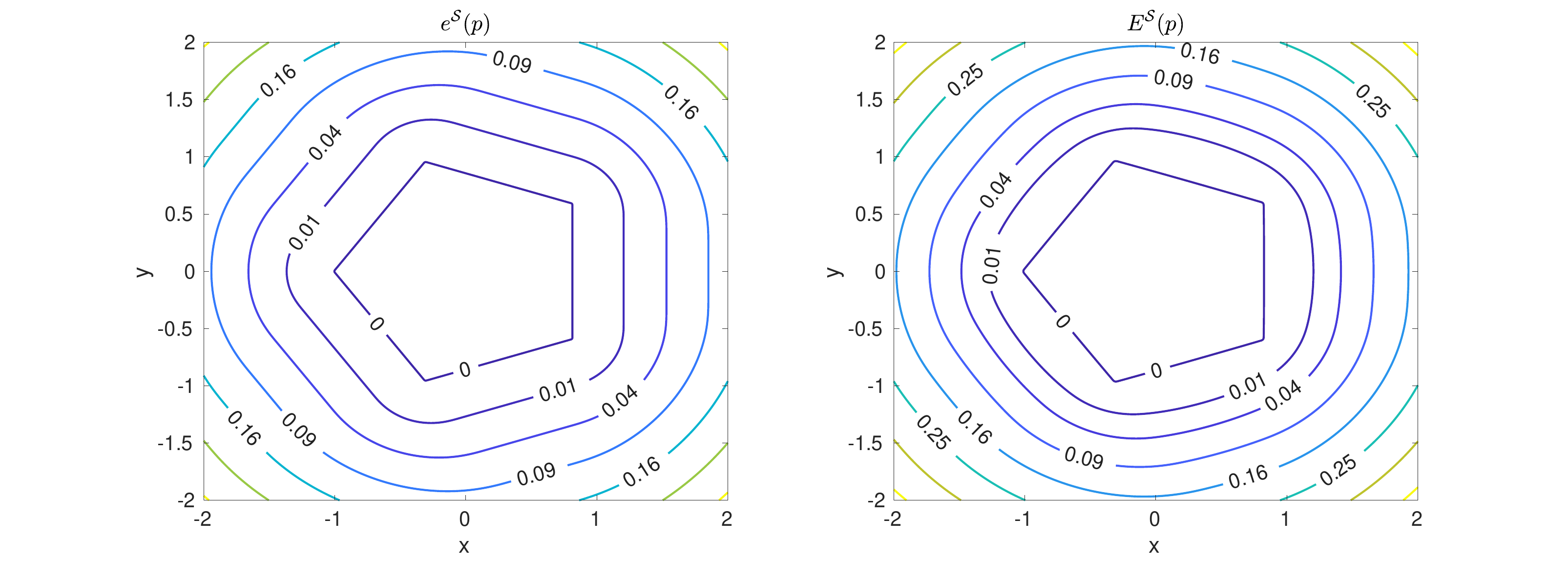}
\caption{Contours of the functions $e^{\mathcal{S}}(p)$ (left) and the strictly convexified $E^{\mathcal{S}}(p)$ (right). In this case, $\mathcal{S}$ is a pentagon (the zero sublevel set of both functions). }
\label{fig:convexified}
\end{figure}

\subsection{Experimental convergence time}
To test the convergence of the generalized alternating algorithm in \eqref{eq:gam}, $50,000$ pairs of random regular convex polytopes, each with 10 inequalities of the form \( u_i^\top p + v_i \leq 0 \), were randomly generated. We considered only those pairs where the true Euclidean distance between them was at least 5 cm. 

Next, the \( k \)-P2S functions were constructed using \( \Phi \) from \eqref{eq:Phifun} with \( k=2 \), \( h=0.1 \), \( e^{\mathcal{S}}(p) \) as defined in \eqref{eq:e} with \( W_i = 1/6 \), and finally \( E^{\mathcal{S}}(p) \) as in \eqref{eq:Efrome} with  \( \varepsilon =0.01\), \( \sigma = 0.989 \), and \( \rho(p) = \|p - p_c\| - R^2 \), where \( p_c \) and \( R \) are chosen such that the ball covers the polytope. 

The iterative algorithm was initialized with \( a[0] \) as the true closest point (in the Euclidean sense) from \( \mathcal{A} \) to \( \mathcal{B} \), computed using GJK's algorithm \cite{gilbert2002fast}. Convergence was considered achieved when \( \|a[k+1] - a[k]\| < 10^{-3} \). The results for the mean computational time (in ms) and the number of iterations are shown in Figure \ref{fig:histogram}. The computational times include the time required to run GJK's algorithm. The computer used was an Intel i7 with a clock speed of 2.30 GHz. The mean computational time was $0.21ms$, with the maximum being $0.67ms$. The mean number of iterations was $27.25$, and the maximum $894$.

\begin{figure}[h]
\includegraphics[trim={0cm 0 0 0},clip, width=8.55cm, height=5.4cm]{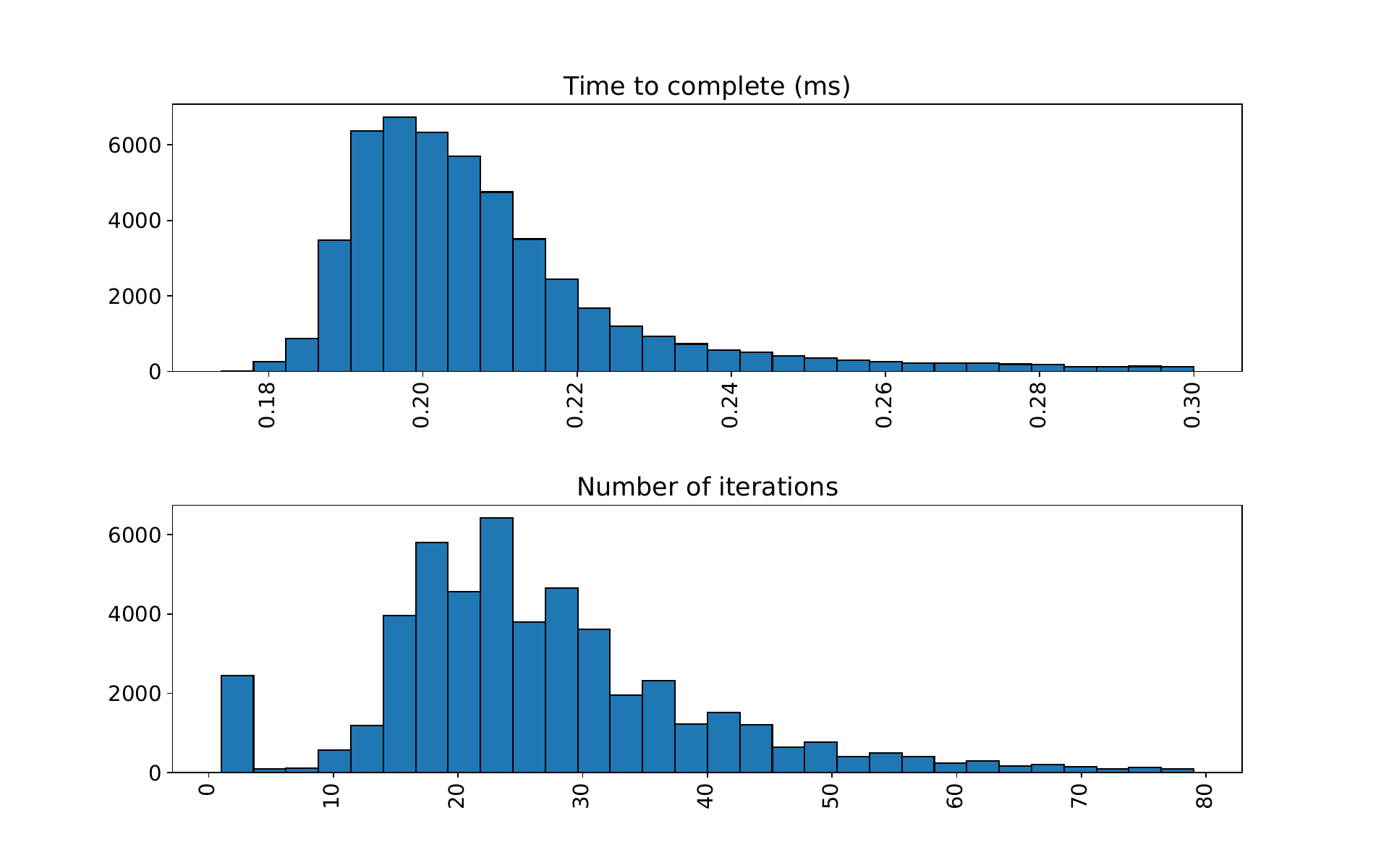}
\caption{Results for the computational experiment. This image can be generated by running the file \texttt{figure\_3.py} in \cite{uaibot_content}.}
\label{fig:histogram}
\end{figure}

\subsection{Experiments with the robot}

To showcase the benefits of using a differentiable distance metric—particularly the one proposed in this paper—we conducted an experiment in which a Franka Emika Panda robot was tasked with reaching a target pose with its end-effector while avoiding collision with a box, self-collision, and respecting joint limits. We implemented a CBF-based controller \cite{ames2019cbf} using both our proposed distance metric and the (half-squared) Euclidean distance, applied to both robot-to-obstacle and self-collision avoidance. In both cases, the robot was controlled via joint velocities, and one CBF inequality was implemented for each pair of objects \((a, b)\), where \(a\) belongs to the robot and \(b\) belongs to the environment (in this case, the obstacle box and the box on which it rests). Figure \ref{fig:expsetup} illustrates the experimental setup through snapshots.

\begin{figure}[h!]
\includegraphics[trim={-5cm 0 0 0},clip, width=8.5cm, height=5.4cm]{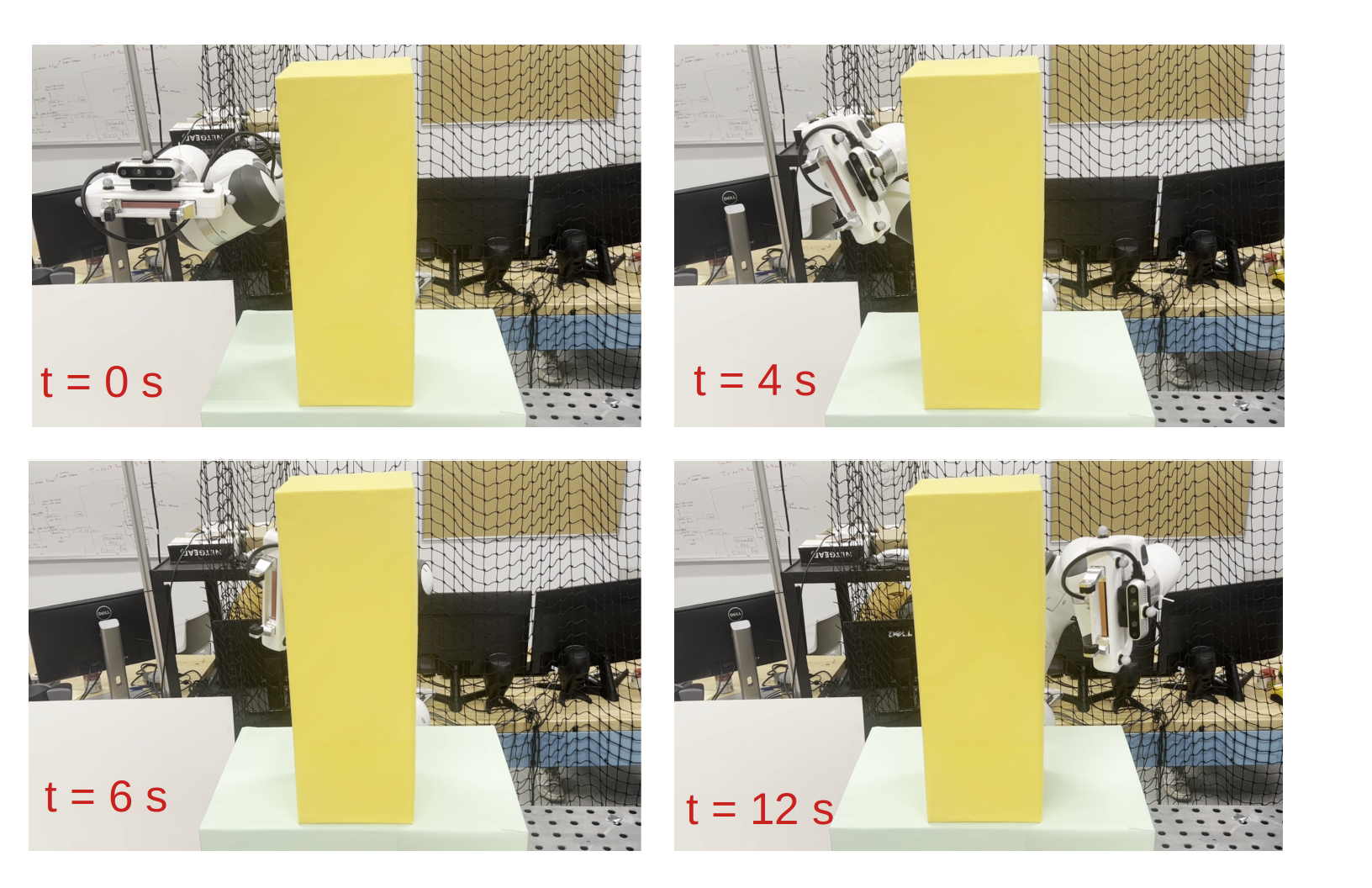}
\caption{Experiment of a Franka Emika 3 reaching a target while avoiding a box. }
\label{fig:expsetup}
\end{figure}

This setup highlights the need for differentiability:  when the gripper of the Franka Emika robot, modeled as a box, becomes parallel to the faces of the obstacle box, the Euclidean distance function exhibits non-differentiability. Specifically, when using the Euclidean distance, we observe that during the portion of the motion where the gripper is nearly parallel to the obstacle, the gradient of the distance (with respect to the robot's configuration) becomes discontinuous. This, in turn, causes abrupt changes in the control input. In contrast, this issue does not arise when using our proposed metric.

Figure \ref{fig:expdata}-top illustrates this phenomenon through the control input for the fifth joint. The chattering behavior is especially noticeable between \(t = 3\,\text{s}\) and \(t = 5\,\text{s}\). Figure \ref{fig:expdata}-bottom shows the Euclidean distance between the gripper and the yellow box within the same critical time window, where clear ``spikes'' caused by non-differentiability can be seen. The same plot also displays the distance produced by our proposed metric, which remains smooth—explaining the smooth behavior of the corresponding control input. The accompanying video provides a more in-depth explanation of this experiment, along with several additional experiments.

\section{Conclusion}
In this paper, we build upon the work \cite{GoncalvesSmoothDistances} and propose a new differentiable distance metric between objects. Compared to the approach in \cite{GoncalvesSmoothDistances}, our method is significantly simpler to implement—particularly for general convex shapes represented by half-spaces, for example. Moreover, it possesses the important feature of vanishing when the objects overlap, which was absent in \cite{GoncalvesSmoothDistances}. We also provide formal proofs of convergence and other theoretical results, which we believe are of independent theoretical interest.

\begin{figure}
\includegraphics[trim={-4cm 0 0 0},clip, width=8cm, height=5.4cm]{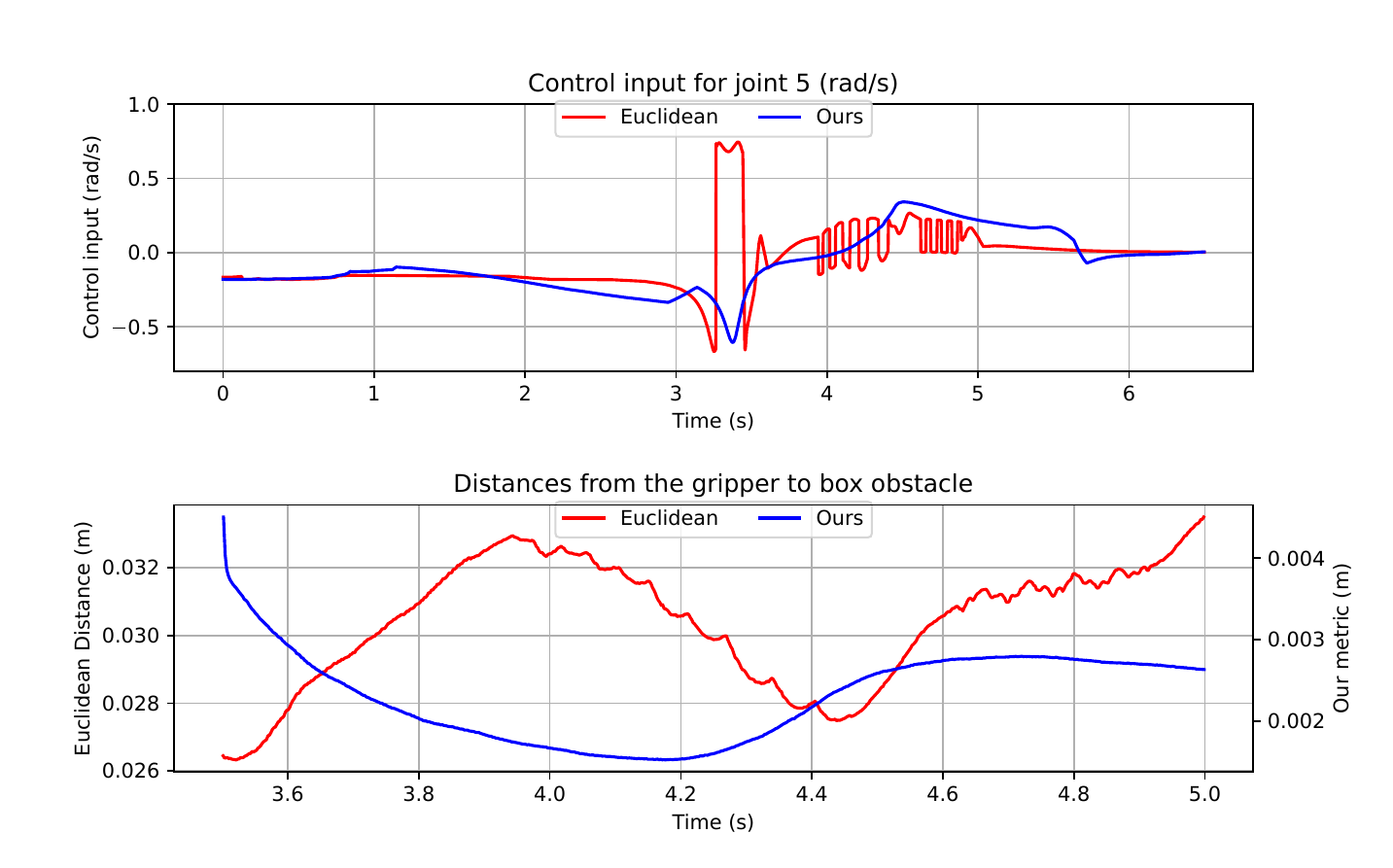}
\caption{Experiment of a Franka Emika 3 avoiding a box. This image can be generated by running the file \texttt{figure\_5.py} in \cite{uaibot_content}.}
\label{fig:expdata}
\end{figure}

\appendix

\textbf{Proof of Proposition \ref{prop:simple}}: from Property (1) in Definition \ref{def:kp2s}, $E^{\mathcal{S}}(p) \geq 0$. From Property (4) in Definition \ref{def:kp2s}, the global minimum of this function is achieved iff $p \in \mathcal{S}$. From Property (2) in Definition \ref{def:kp2s}, $E^{\mathcal{S}}(p)$ is a convex function. Since convex functions only have global minima, and the function is differentiable, the gradient vanishes if and only $p$ lies in this global minima, that is, if $p \in \mathcal{S}$. With this property and the definition of $\Pi^{\mathcal{S}}(p) = p - \frac{\partial E^{\mathcal{S}}}{\partial p}(p)$, it is easy to see that $\Pi^{\mathcal{S}}(p) = p$ iff $p \in \mathcal{S}$. $\square$

\begin{lemma} \label{lemma:cont} Consider the map $F(a) \triangleq \Pi^{\mathcal{A}}\big(\Pi^{\mathcal{B}}(a)\big)$ coming from two k-P2S functions. Then it is locally contractible if $a \not \in \mathcal{A} \cap \mathcal{B} $.
\end{lemma}

\begin{proof} The Jacobian of the  map $F(a) = \Pi^{\mathcal{A}}\big(\Pi^{\mathcal{B}}(a)\big)$ is:

\begin{equation}
    \frac{\partial F}{\partial a}(a) = \frac{\partial \Pi^{\mathcal{B}}}{\partial p}(a)  \frac{\partial \Pi^{\mathcal{A}}}{\partial p}\big(\Pi^{\mathcal{B}}(a)\big).
\end{equation}

Taking the spectral norm $\|\cdot\|$ of both sides and using the submultiplicativity of the spectral norm, we obtain

\begin{equation}
\label{eq:boundeq}
     \left\|\frac{\partial F}{\partial a}(a)\right\| \leq \left\|\frac{\partial \Pi^{\mathcal{B}}}{\partial p}(a) \right\| \left\| \frac{\partial \Pi^{\mathcal{A}}}{\partial p}\big(\Pi^{\mathcal{B}}(a)\big)\right\|.   
\end{equation}


The Jacobian of the map $\Pi^{\mathcal{A}}(p)$ is $I - \frac{\partial^2 E^{\mathcal{A}}}{\partial p^2}(p)$, which is a symmetric matrix and hence has a spectral norm equal to its largest eigenvalue in absolute value. According to Property (2) in Definition \ref{def:kp2s},  this matrix has eigenvalues in the open interval $(0,1)$ if $p \not \in \mathcal{A}$ and eigenvalue $1$ if  $p \in \mathcal{A}$. The same is true, analogously, for $\Pi^{\mathcal{B}}$. So, as long as either $\Pi^{\mathcal{B}}(a) \not \in \mathcal{A}$ or $a \not \in \mathcal{B}$, which is always true if $\mathcal{A} \cap \mathcal{B} = \emptyset$, from \eqref{eq:boundeq} it follows that $\left\|\frac{\partial F}{\partial a}(a)\right\|$ is strictly less than $1$ and the map is guaranteed to be locally contractible.

\end{proof}

\textbf{Proof of Proposition \ref{prop:gan}}:
Lemma \ref{lemma:cont} shows that if $\mathcal{A} \cap \mathcal{B} = \emptyset$, the map $F$ is always locally contractible. This implies global contractivity (\cite{contmap}).  Thus, Banach's fixed point theorem \cite{banach} guarantees that the sequence in \eqref{eq:gam} converges, and that the limit point is unique.
$\square$

\begin{lemma} \label{lemma:inv} Let $\mathcal{S}$ be a regular convex set. Then the function $\Pi^{\mathcal{S}}: \mathbb{R}^n \mapsto \mathbb{R}^n$ is injective, that is, if the equation $\Pi^{\mathcal{S}}(p) = q$ has a solution \(p\) for a given \(q\), it has to be unique.
\end{lemma}
\begin{proof}
    $\Pi^{\mathcal{S}}$ is the gradient of the function $\pi^{\mathcal{S}}(p) = \|p\|^2/2 - E^{\mathcal{S}}(p)$ which, according to Property (2) in Definition \ref{def:kp2s}, is a differentiable strictly convex function, and thus it is an injective function (see \cite{Griewank1991}). 
\end{proof}
\begin{lemma} \label{lemma:gt} Let $\mathcal{A}, \mathcal{B}$ be two regular convex sets in which $\mathcal{A} \cap \mathcal{B} = \emptyset$. If $E^{\mathcal{A}}$ and $E^{\mathcal{B}}$ are two k-P2S functions for two regular convex sets, we have:
\begin{equation}
    \Lambda^{\mathcal{A},\mathcal{B}} = \min_a \max_b \left(  E^{\mathcal{A}}(b)+E^{\mathcal{B}}(a)-\frac{\|a-b\|^2}{2} \right)
\end{equation}
\noindent in which the optimizer $b$ for a given $a$, $b^{\circ}(a)$,  is such that $\Pi^{\mathcal{A}}(b^{\circ}) = a$, whereas the outer optimizer for $a$, $a^{\circ}$ is such that $\Pi^{\mathcal{B}}(a^{\circ}) = b^{\circ}$.
\end{lemma}

\begin{proof}
 This is a game-theoretic interpretation of $\Lambda^{\mathcal{A},\mathcal{B}}$, in which we have a two-player competitive game between player $A$ (minimizer) and player $B$ (maximizer).  Let $\Lambda(a,b) \triangleq E^{\mathcal{A}}(b)+E^{\mathcal{B}}(a)-\frac{\|a-b\|^2}{2}$ and $\Lambda^{\mathcal{B}}(a) \triangleq \max_b \Lambda(a,b)$. We start with player $B$, by maximizing   $\Lambda(a,b)$ over $b$. By differentiating and setting to $0$, we find that the optimal strategy $b^{\circ}(a)$ for player $B$ in accordance to $A$'s choice $a$ is to have $\Pi^{\mathcal{A}}(b^{\circ}(a)) = a$ if this equation has a solution ($a$ is in $\Pi^{\mathcal{A}}$'s range), or set $b^{\circ}(a)$ to an infinite vector following the direction given by the gradient of $\Lambda$ in $b$, namely $a - \Pi^{\mathcal{A}}(b)$, forever to earn an infinite objective function $\Lambda$ ($B$'s gain is unlimited). Since $A$ wants to minimize $\Lambda$, we can assume it is going to choose $a$ such that $\Pi^{\mathcal{A}}(b^{\circ}(a)) = a$ has a solution to prevent this. Furthermore, this equation will have an  unique  solution, that can be found by inverting $\Pi^{\mathcal{A}}$ (see Lemma \ref{lemma:inv}). Finally, this strategy is indeed the maximizer because the Hessian of   $\Lambda(a,b)$  in $b$ is $\frac{\partial^2 E^{\mathcal{A}}}{\partial p^2}(b^{\circ})-I_{n \times n}$, which,    according to Property (2) of Definition \ref{def:kp2s} is negative definite.

 The game proceeds with player $A$'s turn. From the previous player action, $\Lambda^{\mathcal{B}}(a) = \max_b \Lambda(a,b) = \Lambda(a,b^{\circ}(a))$. Differentiating in $a$ and setting to $0$, and using the fact that $\Pi^{\mathcal{A}}(b^{\circ}(a)) = a$ (which is $B$'s strategy), we obtain that the optimal strategy $a^{\circ}$ to $A$ is to have $\Pi^{\mathcal{B}}(a^{\circ}) = b^{\circ}(a^{\circ})$. This also always has a solution $a^{\circ}$. Indeed, applying $\Pi^{\mathcal{A}}$ on both sides (which is an injective mapping, see Lemma \ref{lemma:inv}, and thus this is an ``if and only if'' step) and using the fact that $\Pi^{\mathcal{A}}(b^{\circ}) = a^{\circ}$ (from $B$'s strategy), we have $\Pi^{\mathcal{A}}(\Pi^{\mathcal{B}}(a^{\circ})) = \Pi^{\mathcal{A}}(b^{\circ}(a^{\circ})) = a^{\circ}$. This is the fixed point equation for which Proposition \ref{prop:gan} ensures that the iterative algorithm in \eqref{eq:gam} will converge to its unique solution $a^*$ since $\mathcal{A} \cap \mathcal{B} = \emptyset$. This shows that the optimizer for both $A$ and $B$ are the differentiable witness points $a^*$ and $b^* = \Pi^{\mathcal{B}}(a^*)$, in accordance to Definition \ref{def:smowp}.
 
 Finally, we need to show that $A$'s strategy is indeed a minimizer. This holds because the Hessian of $\Lambda^{\mathcal{B}}(a)$ is
 \begin{equation}
      \frac{\partial^2 E^{\mathcal{B}}}{\partial p^2}(a^{\circ})-I_{n \times n} - \Big(\frac{\partial^2 E^{\mathcal{A}}}{\partial p^2}(b^{\circ}) - I_{n \times n}\Big)^{-1}
 \end{equation}
 \noindent which is positive definite: $\frac{\partial^2 E^{\mathcal{B}}}{\partial p^2}(a^{\circ})$ is positive definite (Property (2) of Definition \ref{def:kp2s}) and 
 $-I_{n \times n} - \Big(\frac{\partial^2 E^{\mathcal{A}}}{\partial p^2}(b^{\circ})-I_{n \times n}\Big)^{-1}$ is as well. To see why for the latter, from Property (2) of Definition \ref{def:kp2s}, the eigenvalues of $H = \frac{\partial^2 E^{\mathcal{A}}}{\partial p^2}(b^{\circ})$ lie on $(0,1)$. Let $f(h) = -1-1/(h-1) = h/(1-h)$. Thus, the eigenvalues of $W = f(H) = -I_{n \times n} -(H-I_{n \times n})^{-1}$ lie on the range $(0,\infty)$ .
 \end{proof}
\begin{lemma} \label{lemma:ineq} Let $\mathcal{A}, \mathcal{B}$ be two regular convex sets in which $\mathcal{A} \cap \mathcal{B} = \emptyset$. Let $(a_0^*,b_0^*)$ be any pair of points that 
achieves the smallest (true) Euclidean distance between them. Then:
\begin{equation}
    E^{\mathcal{B}}(a_0^*) \geq \Lambda^{\mathcal{A},\mathcal{B}} \geq \min_a \Lambda(a,a)
\end{equation}
\noindent in which we borrowed the definition of $\Lambda(a,b)$ from the proof of Lemma \ref{lemma:gt}.
\end{lemma}

\begin{proof}
    We apply the game-theoretic/optimization interpretation of $\Lambda^{\mathcal{A},\mathcal{B}}$ from Lemma \ref{lemma:gt}. For the leftmost inequality, since $\Lambda^{\mathcal{A},\mathcal{B}} = \min_a \Lambda^{\mathcal{B}}(a)$, then $\Lambda^{\mathcal{A},\mathcal{B}} \leq  \Lambda^{\mathcal{B}}(a)$ for any $a$. Take $a =  a_0^*$. In this case, we have
    \begin{equation}
         \Lambda^{\mathcal{B}}(a_0^*) = \Lambda(a^*_0,b^{\circ}(a_0^*))
    \end{equation}
    \noindent in which, according to Lemma \ref{lemma:gt}, $b^{\circ}(a_0^*)$ is such that $\Pi^{\mathcal{A}}(b^{\circ}(a_0^*)) = a_0^*$. Note that, since $a_0^* \in \mathcal{A}$, $b^{\circ}(a_0^*) = a_0^*$ is a solution to this equation (See Proposition \ref{prop:simple}). Furthermore, from Lemma \ref{lemma:inv}, $\Pi^{\mathcal{A}}$ is injective, and thus  $\Pi^{\mathcal{A}}(b^{\circ}(a_0^*)) = a_0^*$ implies $b^{\circ}(a_0^*) = a_0^*$. Hence:
    \begin{equation}
         \Lambda^{\mathcal{B}}(a_0^*) = \Lambda(a^*_0,a^*_0) = E^{\mathcal{A}}(a_0^*) + E^{\mathcal{B}}(a_0^*)-\frac{\|a_0^*-a_0^*\|^2}{2}
    \end{equation}    
    \noindent which is equal to $E^{\mathcal{B}}(a_0^*)$ since $E^{\mathcal{A}}(a_0^*)=0$ (once $a_0^* \in \mathcal{A}$, see Property (4) of Definition \ref{def:kp2s}).
    
    For the rightmost inequality, note that $\Lambda^{\mathcal{B}}(a) = \max_b \Lambda(a,b) \geq \Lambda(a,b)$ for all $b$. Take $b(a) = a$, and thus $\Lambda^{\mathcal{B}}(a)  \geq \Lambda(a,a)$ for any $a$. Taking the minimum in $a$ in both sides and recalling Lemma \ref{lemma:gt}, the result follows.
\end{proof}

\noindent \textbf{Proof of Proposition \ref{prop:pos}}:

\underline{Proof of positivity:} We utilize the definitions of $\Lambda(a,b)$ and $\Lambda^{\mathcal{B}}(a)$ from the proof of Lemma \ref{lemma:gt}. From the rightmost inequality in Lemma \ref{lemma:ineq}, $\Lambda^{\mathcal{A},\mathcal{B}} \geq \min_a \Lambda(a,a)$. But $\Lambda(a,a) = E^{\mathcal{A}}(a) + E^{\mathcal{B}}(a)$, and this is nonnegative (Property (1) of Definition \ref{def:kp2s}). Furthermore, $E^{\mathcal{A}}(a) + E^{\mathcal{B}}(a)$ can be $0$ only if $E^{\mathcal{A}}(a) = E^{\mathcal{B}}(a) = 0$  (Property (1) together with Property (4) of Definition \ref{def:kp2s}). This implies that a point $a$ such that  $a \in \mathcal{A}$ and $a \in \mathcal{B}$ should exist, that is, $\mathcal{A} \cap \mathcal{B} \not= \emptyset$.

\underline{Proof of going to zero when $\mathcal{A} \cap \mathcal{B} \not= \emptyset$:} in that case, we start by using  the leftmost inequality in Lemma \ref{lemma:ineq}. As the two objects tend to overlap, the point $a_0^*$, that is always inside $\mathcal{A}$, tends to get closer to $\mathcal{B}$ as well (since, by definition, $a_0^*$, is the closest point in $\mathcal{A}$ to $\mathcal{B}$). But when $a_0^* \in \mathcal{B}$, we have that $E^{\mathcal{B}}(a_0^*) = 0$ (Proposition \ref{prop:simple}). Finally, $E^{\mathcal{B}}$ is continuous (since it is $k$ times differentiable, $k \geq 2$ from Property (3) in Definition \ref{def:kp2s}). Hence, from both inequalities in Lemma \ref{lemma:ineq}, $\Lambda^{\mathcal{A},\mathcal{B}}$ is squeezed from the left (due to this argument) and from the right (see the ``Proof of positivity'' above) to $0$. And the result follows.
$\square$

\textbf{Proof of Proposition \ref{prop:diff}}: 

\underline{Part I:} we start by showing that if $\tau$ is such that $\mathcal{A}(\tau) \cap \mathcal{B}(\tau) = \emptyset$ and $a^*(\tau)$ is a limit point of \eqref{eq:gam}, that is, $a^*(\tau) = \Pi^{\mathcal{A}(\tau)}\big(\Pi^{\mathcal{B}(\tau)}\big(a^*(\tau)\big)\big)$, then $a^*(\tau)$ is $k$-times differentiable. The proof is by induction.

Let $F(p,\tau) = \Pi^{\mathcal{A}(\tau)}\big(\Pi^{\mathcal{B}(\tau)}\big(p\big)\big)$ and thus the fixed point equation is $a^*(\tau) = F\big(a^*(\tau),\tau\big)$. Let $M(\tau) = I_{n \times n}-\frac{\partial F}{\partial p}(a^*(\tau),\tau)$.  Differentiate $a^*(\tau) = F\big(a^*(\tau),\tau\big)$ $k$-times. We obtain an equation of the form:

\begin{equation}
   M(\tau)\frac{d^{k}a^*}{d\tau^{k}}(\tau) = G\Big(a^*(\tau),\frac{da^*}{d\tau}(\tau),..., \frac{d^{k-1}a^*}{d\tau^{k-1}}(\tau),\tau\Big)
\end{equation}

\noindent in which the function $G: \mathbb{R}^n \times \mathbb{R}^n \times ... \times  \mathbb{R}^n \times \mathbb{R} \mapsto \mathbb{R}^n$ is continuous in all its arguments, because the functions $E^{\mathcal{A(\tau)}}(p)$, $B^{\mathcal{A(\tau)}}(p)$ are $k$ times differentiable on $p$ (Property (3) in Definition \ref{def:kp2s}) and also in $\tau$ because the transformations $T_A(p,\tau), T_B(p,\tau)$ were assumed to be $k$ times differentiable. 

We proceed by noting that it is established, in the proof of Lemma \ref{lemma:cont}, that $\|\frac{\partial F}{\partial p}(a^*(\tau),\tau)\| < 1$ as long as $\mathcal{A} \cap \mathcal{B} = \emptyset$. Hence \(M(\tau)\) is invertible and consequently:

\begin{equation}
   \frac{d^{k}a^*}{d\tau^{k}}(\tau) = M(\tau)^{-1}G\Big(a^*(\tau),\frac{da^*}{d\tau}(\tau),..., \frac{d^{k-1}a^*}{d\tau^{k-1}}(\tau),\tau\Big).
\end{equation}

\noindent which shows that $\frac{d^{k}a^*}{d\tau^{k}}(\tau)$ is continuous as long as $\frac{d^{m}a^*}{d\tau^{m}}(\tau)$ is continuous for $m=0,..,k-1$.  The base case of the induction, the continuity of $a^*(\tau)$ on $\tau$, comes from the fact that the map $F(p,\tau)$ is contractive as long as $\mathcal{A}(\tau) \cap \mathcal{B}(\tau) = \emptyset$ (see \cite{beldzinski2022dependence}).

\underline{Part II:} we need to prove that $b^*(\tau) = \Pi^{\mathcal{B}(\tau)}(a^*(\tau))$ is also $k$ times differentiable. The proof is analogous to the one in Part I, by differentiating this equation $k$ times in $\tau$.

\underline{Part III:} for the final part, we note that the $k^{th}$ derivative $\Lambda^{\mathcal{A}(\tau),\mathcal{B}(\tau)}$ involves the $k^{th}$ derivatives of the functions $E^{\mathcal{A}(\tau)}(p)$ and $E^{\mathcal{B}(\tau)}(p)$ on $p$ and $\tau$. These  are continuous (from Property (3) in Definition \ref{def:kp2s} for $p$ and also in $\tau$ because the transformations $T_A(p,\tau), T_B(p,\tau)$ were assumed to be $k$ times differentiable). Finally, it also depends ont the $k^{th}$ derivatives of $a^*(\tau)$ and $b^*(\tau)$ on $\tau$, which were established to be continuous in Part I and Part II, respectively. $\square$

\begin{lemma} \label{lemma:ABsum} Let $A, B : \mathbb{R}^n \mapsto \mathbb{R}$ be two functions that are at least twice differentiable, in which $A$ is strictly convex everywhere and $B$ is convex everywhere. Furthermore, suppose that $B \geq 0$ and that $A$ and $B$ do not vanish simultaneously. Then:
\begin{equation}
    C(p) = A(p) + \sqrt{ A(p)^2 + B(p)^2}
\end{equation}
\noindent is strictly convex in the set $\mathcal{P} = \{p \in \mathbb{R}^n \ | \ B(p) > 0\}$.
\end{lemma}

\begin{proof}
It can be checked by applying derivation rules that:
\begin{equation}
\label{eq:hessianC}
    \frac{\partial^2 C}{\partial p^2} = \frac{C \frac{\partial^2 A}{\partial p^2} + B \frac{\partial^2 B}{\partial p^2}}{\sqrt{A^2+B^2}}{+}\frac{\left(A \frac{\partial B}{\partial p} {-} B\frac{\partial A}{\partial p}\right)\left(A \frac{\partial B}{\partial p} {-} B\frac{\partial A}{\partial p}\right)^\top}{(\sqrt{A^2+B^2})^3}.
\end{equation}   

    Note that the division by $\sqrt{A^2+B^2}$ never causes problems because $A$ and $B$ cannot simultaneously vanish. Furthermore, $C > 0$ when $p \in \mathcal{P}$. Thus, when $p \in \mathcal{P}$, the first term in \eqref{eq:hessianC} is a positive definite matrix since $B \geq 0$, $\frac{\partial^2 A}{\partial p^2} \geq 0$, $C > 0$ and $\frac{\partial^2 B}{\partial p^2} > 0$. In addition, the second term is also positive semidefinite (a matrix of the form $\alpha uu^\top$ for any vector $u \in \mathbb{R}^n$ and a nonnegative $\alpha$ is always positive semidefinite). Since the sum of a positive definite and a positive semidefinite matrix is always a positive definite matrix, the desired result follows.
\end{proof}

\begin{lemma} \label{lemma:convprop} If $U: \mathbb{R}^n \mapsto \mathbb{R}$ is convex, twice differentiable with $\|\frac{\partial^2 U}{\partial p^2}\| \leq \mu$, and $U_{\textsl{min}} = \min_q U(q)$ is finite, then:
\begin{equation}
    \left\|\frac{\partial U}{\partial p}(p)\right\|^2 \leq 2 \mu (U(p)-U_{\textsl{min}}) \ \ \ \forall p.
\end{equation}
\end{lemma}

\begin{proof} This is a simple application of Descent Lemma \cite{Bertsekas1999}: if  $U: \mathbb{R}^n \mapsto \mathbb{R}$ is convex, twice differentiable with $\|\frac{\partial^2 U}{\partial p^2}\| \leq \mu$, then, for any $q,p$, $U(q) \leq U(p) + (q-p)^\top \nabla U(p) + \frac{\mu}{2}\|p-q\|^2$. After minimizing both sides in $q$ and reorganizing, the result follows.
\end{proof}

\begin{lemma} \label{lemma:bounded} Let $A$ and $B$ be convex twice differentiable functions such that $\|\frac{\partial^2 A}{\partial p}\| \leq \alpha$ and $\|\frac{\partial^2 B}{\partial p}\| \leq \beta$ for constant  $\alpha, \beta > 0$. Furthermore, let $A_{\textsl{min}} = \min_p A(p)$, $B_{\textsl{min}} = \min_p B(p)$ and $V_{\textsl{min}} = \min_p \sqrt{A(p)^2+B(p)^2}$. Then, the Hessian in \eqref{eq:hessianC}  satisfies:
\begin{equation}
\label{eq:boundhessianC}
    \left\|\frac{\partial^2 C}{\partial p^2}\right\| \leq \alpha + \left(3{+}\frac{\sqrt{A_{\textsl{min}}^2+B_{\textsl{min}}^2}}{V_{\textsl{min}}} \right)\sqrt{\alpha^2+\beta^2}. 
\end{equation}
    
\end{lemma}

\begin{proof} Let $V = [A \ B]^{\top}$, $\nu = [\alpha \ \beta]^\top$, $M = [-\frac{\partial B}{\partial p} \ \frac{\partial A}{\partial p}]$ and $V_{\textsl{min}} = [A_{\textsl{min}} \ B_{\textsl{min}}]^{\top}$. We begin by noting that the rightmost element at the right-hand side of \eqref{eq:hessianC} can be written as $V^\top M M V$. Applying the norm in both sides of \eqref{eq:hessianC}, imposing the bounds  $\|\frac{\partial^2 A}{\partial p}\| \leq \alpha$ and $\|\frac{\partial^2 B}{\partial p}\| \leq \beta$ and applying Cauchy-Schwarz's inequality/ submultiplicativity of the spectral norm, we can bound  $\|V^\top M M V\| \leq \|M\|^2\|V\|^2$ and thus:

\begin{equation}
  \left\|\frac{\partial^2 C}{\partial p^2} \right\| \leq 
  \frac{C}{\|V\|}\alpha + \frac{B}{\|V\|} \beta + \frac{\|M\|^2}{\|V\|}.
\end{equation}
Since $C= A + \|V\|$, the previous equality can be written as:
\begin{equation}
  \left\|\frac{\partial^2 C}{\partial p^2} \right\| \leq 
  \alpha + \nu^\top \frac{V}{\|V\|} + \frac{\|M\|^2}{\|V\|}.
\end{equation}
Again applying Cauchy-Schwarz to $\nu^\top \frac{V}{\|V\|}$, we find that it is smaller or equal than than $\|\nu\| = \sqrt{\alpha^2+\beta^2}$, and thus:
\begin{equation}
  \left\|\frac{\partial^2 C}{\partial p^2} \right\| \leq 
  \alpha + \sqrt{\alpha^2+\beta^2} + \frac{\|M\|^2}{\|V\|}.
\end{equation}
We will bound $\frac{\|M\|^2}{\|V\|}$. We note that $\|M\| \leq \|M\|_F$, in which $\|M\|_F$ is the Frobenius norm \cite{Horn1990}, thus
\begin{equation}
   \|M\|^2 \leq \left\|\frac{\partial A}{\partial p} \right\|^2+\left \|\frac{\partial B}{\partial p}\right \|^2. 
\end{equation} From Lemma \ref{lemma:convprop}:
\begin{eqnarray}
&&\frac{\|M\|^2}{\|V\|} \leq 2\alpha\frac{A}{\|V\|}{+}2\beta\frac{B}{\|V\|} {-}2\alpha\frac{A_{\textsl{min}}}{\|V\|}{-}2\beta \frac{B_{\textsl{min}}}{\|V\|}= \nonumber \\
&& 2\frac{V^\top \nu}{\|V\|}-2\frac{V_{\textsl{min}}^\top \nu}{\|V\|}
\end{eqnarray}
Applying Cauchy-Schwarz in $V^\top \nu$ and $V_{\textsl{min}}^\top \nu$:
\begin{equation}
   \frac{\|M\|^2}{\|V\|} \leq \left(2+ \frac{\|V_{\textsl{min}}\|}{\|V\|}\right)\|\nu\| \leq \left(2+ \frac{\|V_{\textsl{min}}\|}{V_{\textsl{min}}}\right)\|\nu\|
\end{equation}
\noindent which concludes the proof.
\end{proof}

\textbf{Proof of Proposition \ref{prop:convexified}}:

We will prove each one of the properties in Definition \ref{def:kp2s}.

\underline{Property (1)}: comes easily from the expression in \eqref{eq:Efrome} and the fact that $\alpha > 0$ and $e^{\mathcal{S}}(p) \geq 0$.

\underline{Property (2)}: from Lemma \ref{lemma:ABsum} using $A = \varepsilon \rho$ and $B =  \sigma e^{\mathcal{S}}$, it is established that $\frac{\partial^2 E^{\mathcal{S}}}{\partial p^2}(p) > 0$. 

To show that $\sigma$ and $\varepsilon$ can be chosen so $I_{n \times n} > \frac{\partial^2 E^{\mathcal{S}}}{\partial p^2}(p)$, we use Lemma \ref{lemma:bounded}  again using $A =  \varepsilon \rho$ and $B = \sigma e^{\mathcal{S}}$. In this case, $\alpha = \varepsilon$ and $\beta = \sigma$ can be taken. It is clear that $\varepsilon$ and $\sigma$ can be taken sufficiently small to make the right hand side of \eqref{eq:boundhessianC} as small as necessary, eventually below $1$. Note that in this case, $A_{\textsl{min}}$ is a finite negative number, $B_{\textsl{min}} = 0$ and $V_{\textsl{min}}$ is a nonzero finite number (since $\varepsilon \rho$ and $\sigma e^{\mathcal{S}}$ cannot vanish simultaneously once $\rho$ needs to be zero when  $e^{\mathcal{S}}$ vanishes). 

\underline{Property (3)}: can be checked by inspection, provided that $e^{\mathcal{S}}$ and $\rho$ are $k$ times differentiable, and $e^{\mathcal{S}}$ and $\rho$ cannot vanish simultaneously.

\underline{Property (4)}:  from the expression in \eqref{eq:Efrome}, it is clear that it is zero only if $e^{\mathcal{S}}(p)=0$. Furthermore, since $\rho$ is negative when $e^{\mathcal{S}}(p)=0$, it indeed vanishes when $p \in \mathcal{S}$. $\square$

\bibliographystyle{ieeetr}

\begin{thebibliography}{10}

\bibitem{GoncalvesSmoothDistances}
V.~M. Gon\c{c}alves, A.~Tzes, F.~Khorrami, and P.~Fraisse, ``Smooth distances
  for second order kinematic robot control,'' {\em IEEE Transactions on
  Robotics}, vol.~40, pp.~2950 -- 2966, 2024.

\bibitem{ames2019cbf}
A.~D. Ames, S.~Coogan, M.~Egerstedt, G.~Notomista, K.~Sreenath, and P.~Tabuada,
  ``Control barrier functions: Theory and applications,'' in {\em Proceedings
  of the 18th European Control Conference (ECC)}, (Naples, Italy), IEEE, June
  2019.
\newblock Conference Date: 25--28 June 2019.

\bibitem{EscandeSCH}
A.~Escande, S.~Miossec, M.~Benallegue, and A.~Kheddar, ``A strictly convex hull
  for computing proximity distances with continuous gradients,'' {\em IEEE
  Transactions on Robotics}, vol.~30, no.~3, pp.~666--678, 2014.

\bibitem{Escande2}
A.~Escande, S.~Miossec, and A.~Kheddar, ``Continuous gradient proximity
  distance for humanoids free-collision optimized-postures,'' in {\em 2007 7th
  IEEE-RAS International Conference on Humanoid Robots}, pp.~188--195, 2007.

\bibitem{capsule}
A.~Dietrich, T.~Wimböck, A.~Albu-Schäeffer, and G.~Hirzinger, ``Integration
  of reactive, torque-based self-collision avoidance into a task hierarchy,''
  {\em IEEE Transactions on Robotics}, vol.~22, pp.~1278--1293, 11 2012.

\bibitem{tracy2022diffpills}
K.~Tracy, T.~A. Howell, and Z.~Manchester, ``Diffpills: Differentiable
  collision detection for capsules and padded polygons.''
  \url{https://arxiv.org/abs/2207.00202}, 2022.

\bibitem{pseudodistance}
X.~Zhu, H.~Ding, and S.~Tso, ``A pseudodistance function and its
  applications,'' {\em IEEE Transactions on Robotics and Automation}, vol.~20,
  no.~2, pp.~344--352, 2004.

\bibitem{pseudodist}
J.~Xu, Z.~Liu, C.~Yang, L.~Li, and Y.~Pei, ``A pseudo-distance algorithm for
  collision detection of manipulators using convex-plane-polygons-based
  representation,'' {\em Robotics and Computer-Integrated Manufacturing},
  vol.~66, p.~101993, 12 2020.

\bibitem{SCHMEIER201567}
A.~Schmeißer, R.~Wegener, D.~Hietel, and H.~Hagen, ``Smooth convolution-based
  distance functions,'' {\em Graphical Models}, vol.~82, pp.~67--76, 2015.

\bibitem{article}
M.~Sanchez, O.~Fryazinov, P.~Fayolle, and A.~Pasko, ``Convolution filtering of
  continuous signed distance fields for polygonal meshes,'' {\em Computer
  Graphics Forum}, vol.~34, p.~277–288, 09 2015.

\bibitem{diff}
K.~Tracy, T.~A. Howell, and Z.~Manchester, ``Differentiable collision detection
  for a set of convex primitives,'' in {\em 2023 IEEE International Conference
  on Robotics and Automation (ICRA)}, pp.~3663--3670, 2023.

\bibitem{randomizedgradient}
H.~J.~T. Suh, T.~Pang, and R.~Tedrake, ``Bundled gradients through contact via
  randomized smoothing,'' {\em IEEE Robotics and Automation Letters}, vol.~7,
  no.~2, pp.~4000--4007, 2022.

\bibitem{bssm}
V.~M. Gonçalves, P.~Krishnamurthy, A.~Tzes, and F.~Khorrami, ``Bssm:
  Gpu-accelerated point-cloud distance metric for motion planning,'' {\em IEEE
  Robotics and Automation Letters}, vol.~9, no.~11, pp.~10319--10326, 2024.

\bibitem{DiffOcclusion}
S.~Wei, B.~Dai, R.~Khorrambakht, P.~Krishnamurthy, and F.~Khorrami,
  ``Diffocclusion: Differentiable optimization based control barrier functions
  for occlusion-free visual servoing,'' {\em IEEE Robotics and Automation
  Letters}, vol.~9, no.~4, pp.~3235--3242, 2024.

\bibitem{contmap}
T.~V. Petersdorff, ``Notes on contractible maps.''
  \url{https://terpconnect.umd.edu/~petersd/666/fixedpoint.pdf}.

\bibitem{vonneumann}
H.~Bauschke and J.~J. Borwein, ``{On the convergence of von Neumann's
  alternating projection algorithm for two sets},'' {\em Set-Valued Analysis},
  vol.~1, pp.~185--212, 06 1993.

\bibitem{gilbert2002fast}
E.~G. Gilbert, D.~W. Johnson, and S.~S. Keerthi, ``A fast procedure for
  computing the distance between complex objects in three-dimensional space,''
  {\em IEEE Journal on Robotics and Automation}, vol.~4, no.~2, pp.~193--203,
  1988.

\bibitem{uaibot_content}
Anonymous. ..., 2025.
\newblock This files are currently available as supplementary material for the
  reviewers, due to anonimity requirements. In the final version, it will be
  available in a URL.

\bibitem{banach}
K.~Ciesielski, ``{On Stefan Banach and some of his results},'' {\em Banach
  Journal of Mathematical Analysis}, vol.~1, no.~1, pp.~1 -- 10, 2007.

\bibitem{Griewank1991}
A.~Griewank, H.~T. Jongen, and M.~K. Kwong, ``The equivalence of strict
  convexity and injectivity of the gradient in bounded level sets,'' {\em
  Mathematical Programming}, vol.~51, pp.~273--278, jul 1991.

\bibitem{beldzinski2022dependence}
M.~Be{\l}dzi{\'n}ski, M.~Galewski, and I.~Kossowski, ``Dependence on parameters
  for nonlinear equations—abstract principles and applications,'' {\em
  Mathematical Methods in the Applied Sciences}, vol.~45, no.~3,
  pp.~1668--1686, 2022.

\bibitem{Bertsekas1999}
D.~P. Bertsekas, {\em Nonlinear Programming}.
\newblock Belmont, Massachusetts: Athena Scientific, 2nd~ed., 1999.

\bibitem{Horn1990}
R.~A. Horn and C.~R. Johnson, {\em Matrix Analysis}.
\newblock Cambridge University Press, 1990.

\end{thebibliography}

\end{document}